\documentclass[runningheads]{llncs}

\usepackage[T1]{fontenc}
\usepackage[utf8]{inputenc}
\usepackage{amsmath,amssymb}
\usepackage{booktabs}
\usepackage{graphicx}
\usepackage{multirow}
\usepackage{url}
\usepackage{enumitem}
\usepackage[hidelinks]{hyperref}

\graphicspath{{fig/}}

\begin{document}

\title{Mitigating Catastrophic Forgetting in Streaming Generative and Predictive Learning via Stateful Replay}
\titlerunning{Stateful Replay for Streaming Learning}

\author{Du Wenzhang}
\authorrunning{Du Wenzhang}
\institute{Dept. of Computer Engineering, Mahanakorn University of Technology, International College (MUTIC)\\Bangkok, Thailand\\\email{dqswordman@gmail.com}}

\maketitle

\begin{abstract}
Neural models deployed on data streams must track a drifting distribution with limited memory.
Sequential fine-tuning on each phase (SeqFT) is simple and architecture-agnostic, but can catastrophically forget earlier regimes when gradients from later phases conflict.
Replay with a finite buffer (Replay) is an equally simple remedy, yet its behaviour across heterogeneous streams and generative objectives is not well understood.

We cast reconstruction, forecasting and classification uniformly as negative log-likelihood minimisation over phase-wise distributions, and quantify forgetting via phase-wise changes in risk.
Viewing SeqFT and Replay as stochastic gradient methods for an ideal joint objective, we give a short gradient-alignment analysis showing how Replay mixes current and historical gradients to turn many ``forgetting steps'' into benign updates.
We then build a mixed generative--predictive benchmark of six streaming scenarios on three public datasets (Rotated MNIST, ElectricityLoadDiagrams2011--2014, Airlines), and log per-phase initial and final metrics to support reproducible analysis.

Across three seeds and matched training budgets, the pattern is sharp:
on heterogeneous multi-task streams (RotMNIST digit pairs, Airlines airline groups), Replay dramatically reduces catastrophic forgetting; on benign time-based streams (Electricity and Airlines), SeqFT and Replay behave similarly and forgetting is negligible.
Code and logs will be released upon publication to facilitate reproducibility.
\end{abstract}

\section{Introduction}
\label{sec:intro}

Many deployed learning systems operate on \emph{streams} rather than static datasets \cite{gama2014survey,bifet2010moa}.
Electricity providers record long-term load curves,
airlines log every flight,
and perception pipelines observe a continuous feed of images and signals.
Models are updated online or in small batches on the newest data, often with strict memory limits.
The most common training protocol is \emph{sequential fine-tuning} (SeqFT): train on phase $1$, then continue optimisation on phase $2$, and so on.

SeqFT reuses standard training code and is architecture-agnostic, but is vulnerable to \emph{catastrophic forgetting}~\cite{mccloskey1989catastrophic,goodfellow2013empirical,parisi2019continual}:
when later phases correspond to different sub-populations, label subsets or tasks, gradients from new phases can overwrite parameters that were useful for earlier ones.
For generative objectives this is especially problematic: once an autoencoder or forecaster no longer reconstructs historical regimes, its outputs stop reflecting the system's history.

Replay with a finite buffer (Replay) is arguably the simplest continual-learning mechanism \cite{lin1992self,isele2018selective}:
keep a small buffer of past examples and mix them into current mini-batches.
Replay is widely used, but empirical reports are mixed.
More sophisticated continual-learning methods based on parameter-importance regularisation, knowledge distillation and generative replay have also been proposed~\cite{kirkpatrick2017ewc,zenke2017synaptic,li2017learning,shin2017continual,lopezpaz2017gradient,kemker2018measuring}, but these often introduce additional complexity and tuning effort.
On some benchmarks replay yields large gains, on others it appears unnecessary.
This raises two basic questions:
\emph{(i) when is replay memory theoretically justified and practically necessary in streaming learning, and (ii) how does its effect differ between generative vs.\ predictive tasks and between heterogeneous vs.\ near-stationary streams?}

\paragraph{Contributions.}
We deliberately focus on a minimalist mechanism---stateful replay with a fixed-capacity buffer---and ask how far it goes when applied consistently across objectives and modalities.
Our contributions are:
\begin{enumerate}[leftmargin=*]
  \item \textbf{Unified streaming formulation.}
  We cast auto-encoding, forecasting and classification uniformly as negative log-likelihood minimisation over phase-wise data distributions, and define a simple phase-wise forgetting functional that applies across metrics.
  \item \textbf{Gradient-alignment view of replay.}
  We interpret SeqFT and Replay as stochastic gradient methods for an ideal joint objective and show that forgetting on a past phase is governed, to first order, by the inner product between its gradient and the update direction.
  Replay replaces the pure current-phase gradient by a mixture of current and historical gradients; when gradients conflict, this mixture turns many ``forgetting steps'' into benign updates.
  \item \textbf{Mixed benchmark with transparent logging.}
  We construct six streaming scenarios on three public datasets:
  Rotated MNIST (reconstruction and digit-pair classification),
  ElectricityLoadDiagrams2011--2014 (one-step load forecasting under time- and meter-group splits),
  and Airlines (delay prediction under time- and airline-group splits).
  All per-phase initial and final metrics for both methods and three seeds are stored in a single structured log, which we use to generate all tables and figures.
  \item \textbf{Empirical characterisation.}
  With matched training budgets, Replay dramatically reduces catastrophic forgetting on genuinely interfering streams (digit pairs and airline groups), while behaving like SeqFT on benign time-based streams where drift is mild.
\end{enumerate}

Our goal is not to propose a complex new algorithm, but to position stateful replay as a strong, theoretically interpretable and well-documented baseline for streaming generative and predictive learning.

\section{Setup and Methods}
\label{sec:setup_method}

\subsection{Streaming generative formulation}

We observe phases $t = 1,\dots,T$, each associated with a distribution $P_t$ over pairs $(x,y)$ and a finite sample
\[
  D_t = \{(x_i^{(t)}, y_i^{(t)})\}_{i=1}^{n_t} \sim P_t^{n_t}.
\]
We consider a model $f_\theta$ with parameters $\theta \in \mathbb{R}^d$ and loss
\begin{equation}
  \label{eq:loss_def}
  \ell(f_\theta(x), y) = -\log q_\theta(y \mid x),
\end{equation}
where $q_\theta$ is a conditional density or mass function.
This covers all tasks considered:
\begin{itemize}[leftmargin=*]
  \item \textbf{Reconstruction} (RotMNIST): $y=x$, $q_\theta(\cdot\mid x)$ is Gaussian with mean $f_\theta(x)$; we evaluate by MSE.
  \item \textbf{Forecasting} (Electricity): $x$ is a past window, $y$ the next time step; we evaluate by MSE.
  \item \textbf{Classification} (RotMNIST, Airlines): $y\in\{1,\dots,C\}$, $q_\theta(y\mid x)$ is softmax; we evaluate by accuracy but train with cross-entropy.
\end{itemize}

The population risk on phase $t$ and the ideal joint risk are
\begin{equation}
  \label{eq:risk_defs}
  R_t(\theta) = \mathbb{E}_{(x,y)\sim P_t}[\ell(f_\theta(x),y)],
  \qquad
  R_{\mathrm{joint}}(\theta) = \frac{1}{T}\sum_{t=1}^T R_t(\theta).
\end{equation}
In an offline setting one could pool all $D_t$ and run SGD on their union.
In streaming we must update $\theta$ as phases arrive, without revisiting the full history.

\subsection{Phase-wise forgetting}

Let $\theta_t$ denote the parameters after finishing phase $t$.
For each phase $k$ we distinguish:
\begin{itemize}[leftmargin=*]
  \item \emph{Initial performance:} empirical risk $\hat{R}_k(\theta_k)$ on a held-out validation split from $P_k$;
  \item \emph{Final performance:} $\hat{R}_k(\theta_T)$ on the same split after training on all $T$ phases.
\end{itemize}
We define phase-wise forgetting as
\begin{equation}
  \label{eq:forgetting_def}
  F_k =
  \begin{cases}
    \hat{R}_k(\theta_T) - \hat{R}_k(\theta_k), & \text{for loss metrics},\\[2pt]
    s_k^{\mathrm{init}} - s_k^{\mathrm{final}}, & \text{for accuracy},
  \end{cases}
\end{equation}
where $s_k^{\mathrm{init}}$ and $s_k^{\mathrm{final}}$ are initial and final accuracies.
Thus $F_k>0$ indicates forgetting, while $F_k<0$ indicates positive backward transfer.
We summarise each dataset--scenario by the average $\bar{F} = \tfrac{1}{T}\sum_k F_k$.

\subsection{SeqFT and stateful replay}

\paragraph{Sequential fine-tuning (SeqFT).}
SeqFT processes phases in order.
At phase $t$ it runs mini-batch SGD on
\begin{equation}
  \hat{R}_t(\theta) = \frac{1}{n_t}\sum_{(x,y)\in D_t} \ell(f_\theta(x),y),
\end{equation}
starting from $\theta_{t-1}$ and producing $\theta_t$.
Let $g_t(\theta)=\nabla_\theta \hat{R}_t(\theta)$ and $\tilde{g}_t$ its mini-batch estimate; an SGD step is $\theta\leftarrow\theta-\eta_t \tilde{g}_t(\theta)$.

\paragraph{Stateful replay (Replay).}
Replay augments SeqFT with an episodic buffer $\mathcal{B}$ of capacity $C$ storing past examples.
After finishing phase $t$, we insert a subset of $D_t$ into $\mathcal{B}$ and evict oldest entries to maintain capacity (reservoir-style~\cite{vitter1985random}).
During phase $t>1$, each update uses a mixed mini-batch: we draw $B$ examples from $D_t$ and, if $\mathcal{B}$ is non-empty, $B$ from $\mathcal{B}$.
Let $\lambda\in[0,1]$ denote the fraction of buffer samples (we use $\lambda\approx0.5$).
The expected gradient at phase $t$ is
\begin{equation}
  \label{eq:replay_grad}
  g_t^{\mathrm{rep}}(\theta)
  = (1-\lambda)\nabla R_t(\theta) + \lambda \nabla R_{\mathcal{B}}^{(t)}(\theta),
\end{equation}
where $R_{\mathcal{B}}^{(t)}$ is the risk under the buffer distribution at phase $t$.
The state at the start of phase $t$ is $(\theta_{t-1},\mathcal{B}_{t-1})$, hence \emph{stateful} replay.

\section{Gradient Alignment View}
\label{sec:theory}

We now sketch a simple gradient-alignment view that connects SeqFT and Replay to phase-wise forgetting.
The aim is to explain the empirical trends we observe, not to provide a fully non-asymptotic theory for deep networks.

\subsection{One-step forgetting and alignment}

Fix a past phase $k<t$ and consider a small parameter update at phase $t$:
$\theta' = \theta - \eta d$ with step direction $d$ and small $\eta>0$.
A first-order expansion of $R_k$ gives
\begin{equation}
  \label{eq:one_step}
  R_k(\theta') \approx R_k(\theta) - \eta \,\langle \nabla R_k(\theta), d \rangle.
\end{equation}
Thus the sign of $\langle \nabla R_k, d \rangle$ determines whether this step helps or hurts phase $k$.

For SeqFT, $d\approx\nabla R_t(\theta)$.
Define the cosine similarity between phases $k$ and $t$ as
\[
  \cos\phi_{k,t}(\theta)
  = \frac{\langle \nabla R_k(\theta), \nabla R_t(\theta)\rangle}
         {\|\nabla R_k(\theta)\|\,\|\nabla R_t(\theta)\|}.
\]
If $\cos\phi_{k,t}>0$, small steps for phase $t$ also decrease $R_k$ (positive backward transfer).
If $\cos\phi_{k,t}<0$, gradients are conflicting and training on phase $t$ increases $R_k$ (local forgetting).
Our heterogeneous streams (digit-pair and airline-group splits) are precisely in this regime, while time-based streams tend to have non-negative cosines.

\subsection{Replay as gradient mixing}

Replay replaces the pure current gradient by the mixture~\eqref{eq:replay_grad}.
Assume the buffer approximates the empirical mixture of past phases, so that
\[
  \nabla R_{\mathcal{B}}^{(t)}(\theta) \approx \bar{g}_{<t}(\theta)
  := \frac{1}{t-1}\sum_{j=1}^{t-1}\nabla R_j(\theta),
\]
and define
\[
  d^{\mathrm{rep}} = (1-\lambda)\nabla R_t(\theta) + \lambda \bar{g}_{<t}(\theta).
\]

\begin{proposition}[Alignment condition]
\label{prop:alignment}
Fix $k<t$ and $\theta$.
Assume:
\begin{enumerate}[label=(\roman*),leftmargin=*]
  \item \textbf{Conflict with the current phase:}
    $\langle \nabla R_k(\theta), \nabla R_t(\theta)\rangle < 0$;
  \item \textbf{Benign historical mixture:}
    $\langle \nabla R_k(\theta), \bar{g}_{<t}(\theta)\rangle \ge 0$.
\end{enumerate}
Then there exists $\lambda^\star \in (0,1)$ such that for all $\lambda \in [\lambda^\star,1]$,
\[
  \langle \nabla R_k(\theta), d^{\mathrm{rep}}\rangle \ge 0,
\]
so the first-order change in $R_k$ under a Replay step is non-positive.
\end{proposition}

\begin{proof}
Sketch.
Let
\[
  h(\lambda) = \langle \nabla R_k, (1-\lambda)\nabla R_t + \lambda \bar{g}_{<t}\rangle.
\]
By (i) we have $h(0)<0$, by (ii) we have $h(1)\ge0$.
Since $h$ is affine in $\lambda$, it has a root $\lambda^\star\in(0,1)$ and $h(\lambda)\ge0$ for all $\lambda\ge\lambda^\star$.
Substituting into~\eqref{eq:one_step} yields the claim.\qed
\end{proof}

In words, when gradients from the current phase conflict with those of a past phase while the historical mixture is benign for that phase, Replay can flip a forgetting step into a non-forgetting step.
This is exactly the regime we see on RotMNIST digit pairs and airline-group streams.

\subsection{Finite buffer approximation}

In practice the buffer contains a finite number $C$ of examples rather than the full past stream.
If the gradient of a single loss is bounded by $\|\nabla_\theta \ell(f_\theta(x),y)\|\le G$, standard concentration bounds show that, at phase $t$, the buffer gradient deviates from $\bar{g}_{<t}$ by at most $O(G/\sqrt{C})$ with high probability.
With a few hundred buffered examples this approximation error is small, and we observe empirically that Replay is robust as long as the buffer is not extremely tiny.

\section{Benchmarks and Protocol}
\label{sec:benchmarks}

\begin{table}[t]
  \centering
  \small
  \caption{Datasets and streaming scenarios used in our experiments.}
  \label{tab:datasets}
  \begin{tabular}{lllclcl}
  \toprule
  Dataset & Split & Task & Metric & \# Phases & \# Seeds & Methods \\
  \midrule
  RotMNIST & digits\_pairs & classification & acc & 5 & 3 & SeqFT, Replay \\
  RotMNIST & digits\_pairs & reconstruction & mse & 5 & 3 & SeqFT, Replay \\
  Electricity & time & forecasting & mse & 5 & 3 & SeqFT, Replay \\
  Electricity & meters & forecasting & mse & 5 & 3 & SeqFT, Replay \\
  Airlines & time & classification & acc & 5 & 3 & SeqFT, Replay \\
  Airlines & airline\_group & classification & acc & 5 & 3 & SeqFT, Replay \\
  \bottomrule
  \end{tabular}
\end{table}

We evaluate SeqFT and Replay on three datasets and six streaming scenarios (Table~\ref{tab:datasets}), covering classification, reconstruction and forecasting.

\paragraph{Rotated MNIST (RotMNIST).}
RotMNIST contains rotated $28\times28$ grayscale digits derived from the MNIST dataset~\cite{lecun1998gradient} and widely used in studies of catastrophic forgetting~\cite{goodfellow2013empirical}.
We define five phases by grouping digits into disjoint pairs:
$\{0,1\}$, $\{2,3\}$, $\{4,5\}$, $\{6,7\}$, $\{8,9\}$.
For reconstruction we train a convolutional autoencoder; for classification we share the encoder and attach a linear classifier head that always predicts all $10$ digits, making phases strongly interfering.

\paragraph{Electricity.}
ElectricityLoadDiagrams2011--2014~\cite{dua2019uci} contains hourly load measurements for 370 customers.
We normalise each series, form sliding windows of length $96$, and forecast the next step with MSE.
We consider:
(i) \emph{time}: five contiguous temporal segments; and
(ii) \emph{meters}: five disjoint groups of customers, each with the full time span.

\paragraph{Airlines.}
The Airlines dataset~\cite{bifet2010moa} consists of over half a million flights with features such as carrier ID, origin and destination airports, day-of-week, scheduled departure time and duration.
The label is a binary delay indicator.
We define:
(i) \emph{time}: five chronological slices; and
(ii) \emph{airline\_group}: five groups of carriers with distinct delay patterns.

\paragraph{Models and training.}
RotMNIST uses a CNN encoder--decoder for reconstruction and a linear classifier head for digits; Electricity uses a small 1D CNN/GRU forecaster; Airlines uses a 3-layer MLP over normalised tabular features.
All models are implemented in PyTorch~\cite{paszke2019pytorch}, trained with Adam~\cite{kingma2015adam} and batch sizes 128--256.
For each dataset--scenario we fix the number of epochs per phase and learning rate based on preliminary tuning, then apply the same schedule to SeqFT and Replay so that training budgets match.
Replay uses a buffer of order $10^3$ examples and a replay ratio $\lambda\approx0.5$.
We run three seeds $\{13,21,42\}$ for each method and scenario.

\paragraph{Metrics and logging.}
For each method, seed and phase $k$ we log initial and final metrics (accuracy or MSE) on a held-out validation split, together with dataset, scenario and method identifiers.
These logs are then aggregated to produce all tables and figures.

\section{Results}
\label{sec:results}

We now summarise the main empirical findings, with per-phase numbers reported in the following tables.

\subsection{Classification streams}
\label{subsec:cls_results}

\paragraph{RotMNIST digit-pair classification.}
Figure~\ref{fig:rotmnist_cls} plots per-phase initial and final accuracy for SeqFT and Replay.
SeqFT achieves high initial accuracy on each digit pair but catastrophically forgets early phases: the first pair loses dozens of accuracy points by the end of the stream, and phases $2$--$3$ also suffer large drops.
Replay preserves a large fraction of early-phase performance while matching SeqFT on the last phase.
Table~\ref{tab:rotmnist_cls} reports per-phase initial/final accuracy and forgetting $F_k$.

\begin{figure}[t]
  \centering
  \includegraphics[width=.8\linewidth]{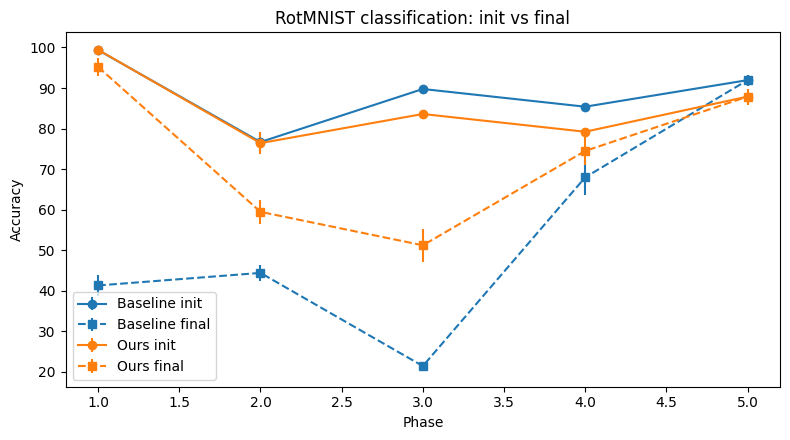}
  \caption{RotMNIST digit-pair classification.
  For each phase we plot initial (solid) and final (dashed) test accuracy for SeqFT and Replay, averaged over seeds.
  SeqFT heavily forgets early digit pairs, while Replay preserves most of their performance.}
  \label{fig:rotmnist_cls}
\end{figure}

\begin{table}[t]
  \centering
  \small
  \caption{RotMNIST digit-pair classification: per-phase initial and final accuracy (\%) and forgetting $F_k$ (init--final, in accuracy points), averaged over three seeds.}
  \label{tab:rotmnist_cls}
  \begin{tabular}{c ccc ccc}
  \toprule
    & \multicolumn{3}{c}{SeqFT} & \multicolumn{3}{c}{Replay (ours)} \\
  \cmidrule(r){2-4}\cmidrule(l){5-7}
  Phase & Init (\%) & Final (\%) & $F_k$ & Init (\%) & Final (\%) & $F_k$ \\
  \midrule
  1 & \(99.4 \pm 0.2\) & \(41.3 \pm 2.7\) & \(58.0 \pm 2.5\) & \(99.4 \pm 0.3\) & \(95.2 \pm 2.2\) & \(\mathbf{4.2} \pm 2.2\) \\
  2 & \(76.7 \pm 2.4\) & \(44.4 \pm 1.9\) & \(32.3 \pm 3.2\) & \(76.4 \pm 2.8\) & \(59.5 \pm 2.9\) & \(\mathbf{16.9} \pm 0.4\) \\
  3 & \(89.8 \pm 0.8\) & \(21.5 \pm 1.1\) & \(68.3 \pm 0.5\) & \(83.6 \pm 0.3\) & \(51.2 \pm 4.1\) & \(\mathbf{32.4} \pm 3.8\) \\
  4 & \(85.4 \pm 0.4\) & \(68.0 \pm 4.2\) & \(17.4 \pm 4.5\) & \(79.2 \pm 0.7\) & \(74.5 \pm 3.5\) & \(\mathbf{4.7} \pm 2.7\) \\
  5 & \(91.9 \pm 1.3\) & \(91.9 \pm 1.3\) & \(0.0 \pm 0.0\) & \(87.9 \pm 2.0\) & \(87.9 \pm 2.0\) & \(0.0 \pm 0.0\) \\
  \bottomrule
  \end{tabular}
\end{table}

\paragraph{Airlines airline-group classification.}
Figure~\ref{fig:airlines_group} shows initial and final accuracy for the heterogeneous carrier-group stream.
SeqFT overfits the last group and severely forgets the first group; Replay roughly halves this loss while keeping later phases competitive.
Table~\ref{tab:airlines_group} details per-phase initial/final accuracy and forgetting.
These patterns align with Proposition~\ref{prop:alignment}: gradients from distinct carrier groups conflict, and replayed samples from earlier groups stabilise the updates.

\begin{figure}[t]
  \centering
  \includegraphics[width=.8\linewidth]{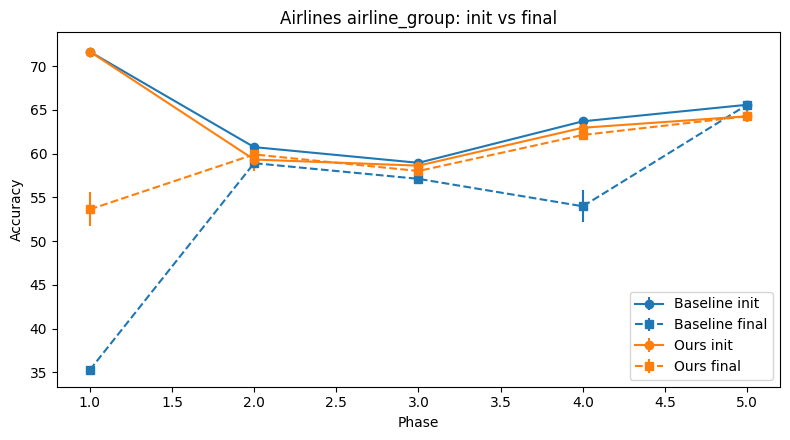}
  \caption{Airlines airline-group classification.
  Replay consistently reduces forgetting on early carrier groups compared to SeqFT, while remaining competitive on later groups.}
  \label{fig:airlines_group}
\end{figure}

\begin{table}[t]
  \centering
  \small
  \caption{Airlines airline-group classification: per-phase initial and final accuracy (\%) and forgetting $F_k$, averaged over seeds.}
  \label{tab:airlines_group}
  \begin{tabular}{c ccc ccc}
  \toprule
    & \multicolumn{3}{c}{SeqFT} & \multicolumn{3}{c}{Replay (ours)} \\
  \cmidrule(r){2-4}\cmidrule(l){5-7}
  Phase & Init (\%) & Final (\%) & $F_k$ & Init (\%) & Final (\%) & $F_k$ \\
  \midrule
  1 & \(71.6 \pm 0.4\) & \(35.3 \pm 0.1\) & \(36.4 \pm 0.4\) & \(71.7 \pm 0.3\) & \(53.6 \pm 1.9\) & \(\mathbf{18.0} \pm 1.7\) \\
  2 & \(60.7 \pm 0.3\) & \(58.9 \pm 0.8\) & \(1.8 \pm 1.0\) & \(59.3 \pm 1.4\) & \(59.9 \pm 0.4\) & \(\mathbf{-0.6} \pm 1.2\) \\
  3 & \(58.9 \pm 0.3\) & \(57.1 \pm 0.4\) & \(1.8 \pm 0.6\) & \(58.6 \pm 0.3\) & \(58.0 \pm 0.4\) & \(\mathbf{0.6} \pm 0.2\) \\
  4 & \(63.7 \pm 0.3\) & \(54.0 \pm 1.8\) & \(9.7 \pm 2.0\) & \(63.0 \pm 0.5\) & \(62.1 \pm 0.3\) & \(\mathbf{0.8} \pm 0.2\) \\
  5 & \(65.6 \pm 0.4\) & \(65.6 \pm 0.4\) & \(0.0 \pm 0.0\) & \(64.3 \pm 0.7\) & \(64.3 \pm 0.7\) & \(0.0 \pm 0.0\) \\
  \bottomrule
  \end{tabular}
\end{table}

\paragraph{Airlines time-based classification.}
On the temporal airline split, both methods show near-zero or slightly negative forgetting.
Later phases act as additional regularisation for the classifier, and final accuracy on early phases can even improve slightly.
This matches the benign drift regime where cross-phase gradients are largely aligned and replay has little to contract.

\subsection{Reconstruction and forecasting}
\label{subsec:recon_forecast}

\paragraph{Electricity forecasting.}
Figure~\ref{fig:electricity_mse} plots initial and final MSE per phase for Electricity under time and meter splits.
Curves for SeqFT and Replay almost overlap; in many cases final MSE is slightly lower than initial, indicating positive transfer.
Per-phase initial and final MSE and the corresponding forgetting values $F_k$ follow the same pattern.
Overall, forgetting is negligible or slightly negative for both methods, consistent with the view that these streams resemble non-stationary single-task training.

\begin{figure}[t]
  \centering
  \includegraphics[width=.9\linewidth]{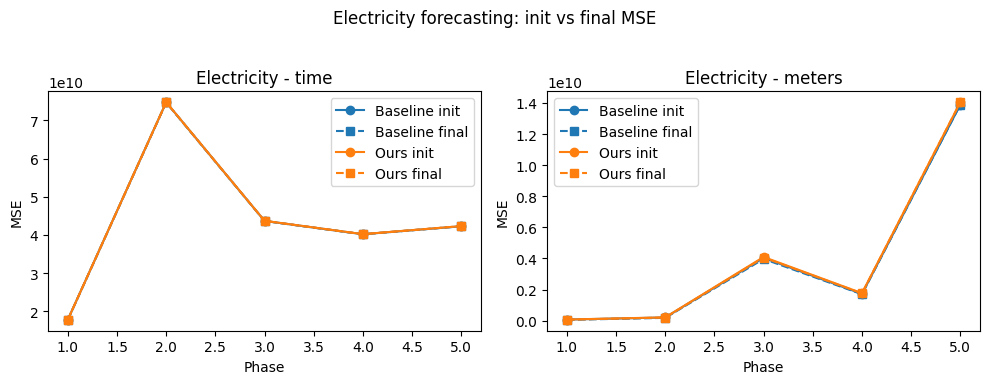}
  \caption{Electricity forecasting under temporal (left) and meter-group (right) splits.
  Initial and final MSE per phase are nearly identical for SeqFT and Replay, indicating negligible forgetting and some positive transfer.}
  \label{fig:electricity_mse}
\end{figure}

\paragraph{RotMNIST reconstruction.}
Per-phase reconstruction MSE on the RotMNIST digit-pair stream shows that both SeqFT and Replay often exhibit \emph{negative} forgetting (final MSE lower than initial), reflecting strong shared structure between digit pairs:
later phases serve as additional regularisation rather than conflicting tasks.

\subsection{Aggregate forgetting}
\label{subsec:aggregate}

To summarise classification behaviour we compute, for each dataset--scenario and method, the mean forgetting $\bar{F}$ across phases and seeds.
Figure~\ref{fig:forgetting_summary} and Table~\ref{tab:avg_forgetting} present these aggregates.

\begin{figure}[t]
  \centering
  \includegraphics[width=.7\linewidth]{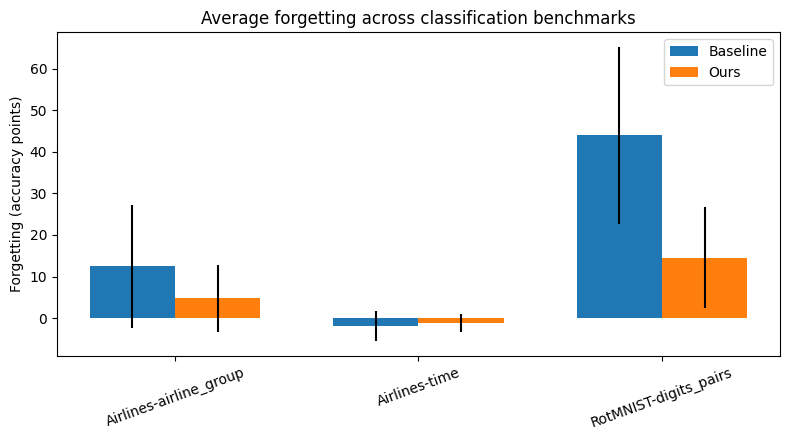}
  \caption{Average forgetting on classification scenarios (SeqFT vs.\ Replay).
  Bars show mean forgetting $F_k$ (init--final, in accuracy points) over phases and seeds; error bars indicate standard deviation.}
  \label{fig:forgetting_summary}
\end{figure}

\begin{table}[t]
  \centering
  \small
  \caption{Average forgetting on classification benchmarks (mean over seeds and phases).
  Forgetting is reported in accuracy points (init--final).}
  \label{tab:avg_forgetting}
  \begin{tabular}{lllc}
  \toprule
  Dataset & Split & Method & Avg forgetting $F$ \\
  \midrule
  RotMNIST & digits\_pairs & SeqFT & \(35.2 \pm 28.2\) \\
  RotMNIST & digits\_pairs & Replay & \(\mathbf{11.7 \pm 13.2}\) \\
  Airlines & time & SeqFT & \(-1.5 \pm 3.4\) \\
  Airlines & time & Replay & \(\mathbf{-1.0 \pm 2.0}\) \\
  Airlines & airline\_group & SeqFT & \(10.0 \pm 15.2\) \\
  Airlines & airline\_group & Replay & \(\mathbf{3.8 \pm 8.0}\) \\
  \bottomrule
  \end{tabular}
\end{table}

On heterogeneous multi-task streams (RotMNIST digit pairs, Airlines airline-group), SeqFT exhibits large positive forgetting, while Replay reduces $|\bar{F}|$ by roughly a factor of two to three.
On benign time-based streams, average forgetting is close to zero for both methods, and Replay behaves like a mild regulariser.

\section{Conclusion}
\label{sec:conclusion}

We presented a unified study of stateful replay for streaming generative and predictive learning.
Formulating reconstruction, forecasting and classification under a single negative log-likelihood objective, and analysing SeqFT and Replay through gradient alignment, we showed how a minimal replay buffer can substantially mitigate catastrophic forgetting on genuinely interfering task streams, while behaving similarly to sequential fine-tuning on benign time-based streams.
Experiments on RotMNIST, Electricity and Airlines confirmed this dichotomy.

Our study suggests that a carefully implemented replay buffer is a strong first-line tool for continual learning in streaming systems.
Future work includes more principled buffer construction and sampling policies, combining replay with parameter-regularisation methods, and scaling the framework to larger architectures and multi-modal streams under realistic resource constraints.


\begin{thebibliography}{99}

\bibitem{bifet2010moa}
Bifet, A., Holmes, G., Kirkby, R., Pfahringer, B.:
MOA: Massive Online Analysis for data stream mining.
J. Mach. Learn. Res. \textbf{11}, 1601--1604 (2010)

\bibitem{dua2019uci}
Dua, D., Graff, C.:
UCI Machine Learning Repository.
University of California, Irvine (2017).
\url{http://archive.ics.uci.edu/ml} (last accessed 2025/11/22)

\bibitem{gama2014survey}
Gama, J., \v{Z}liobait\.{e}, I., Bifet, A., Pechenizkiy, M., Bouchachia, A.:
A survey on concept drift adaptation.
ACM Comput. Surv. \textbf{46}(4), 44:1--44:37 (2014).
DOI: 10.1145/2523813

\bibitem{goodfellow2013empirical}
Goodfellow, I.J., Mirza, M., Xiao, D., Courville, A., Bengio, Y.:
An empirical investigation of catastrophic forgetting in gradient-based neural networks.
In: International Conference on Learning Representations (ICLR 2014) (2014).
DOI: 10.48550/arXiv.1312.6211

\bibitem{isele2018selective}
Isele, D., Cosgun, A.:
Selective experience replay for lifelong learning.
In: AAAI Conference on Artificial Intelligence (AAAI-18), pp. 3302--3309 (2018).
DOI: 10.1609/aaai.v32i1.11595

\bibitem{kemker2018measuring}
Kemker, R., McClure, M., Abitino, A., Hayes, T.L., Kanan, C.:
Measuring catastrophic forgetting in neural networks.
In: AAAI Conference on Artificial Intelligence (AAAI-18), pp. 3390--3398 (2018).
DOI: 10.1609/aaai.v32i1.11651

\bibitem{kingma2015adam}
Kingma, D.P., Ba, J.:
Adam: A method for stochastic optimization.
In: International Conference on Learning Representations (ICLR 2015) (2015).
DOI: 10.48550/arXiv.1412.6980

\bibitem{kirkpatrick2017ewc}
Kirkpatrick, J., et al.:
Overcoming catastrophic forgetting in neural networks.
Proc. Natl. Acad. Sci. USA \textbf{114}(13), 3521--3526 (2017).
DOI: 10.1073/pnas.1611835114

\bibitem{lecun1998gradient}
LeCun, Y., Bottou, L., Bengio, Y., Haffner, P.:
Gradient-based learning applied to document recognition.
Proc. IEEE \textbf{86}(11), 2278--2324 (1998).
DOI: 10.1109/5.726791

\bibitem{li2017learning}
Li, Z., Hoiem, D.:
Learning without forgetting.
IEEE Trans. Pattern Anal. Mach. Intell. \textbf{40}(12), 2935--2947 (2018).
DOI: 10.1109/TPAMI.2017.2773081

\bibitem{lin1992self}
Lin, L.-J.:
Self-improving reactive agents based on reinforcement learning, planning and teaching.
Mach. Learn. \textbf{8}(3--4), 293--321 (1992).
DOI: 10.1007/BF00992699

\bibitem{lopezpaz2017gradient}
Lopez-Paz, D., Ranzato, M.:
Gradient episodic memory for continual learning.
In: Advances in Neural Information Processing Systems 30 (NeurIPS 2017), pp. 6467--6476 (2017).
DOI: 10.48550/arXiv.1706.08840

\bibitem{mccloskey1989catastrophic}
McCloskey, M., Cohen, N.J.:
Catastrophic interference in connectionist networks: The sequential learning problem.
In: Bower, G.H. (ed.) Psychology of Learning and Motivation, vol. 24, pp. 109--165.
Academic Press, San Diego (1989).
DOI: 10.1016/S0079-7421(08)60536-8

\bibitem{parisi2019continual}
Parisi, G.I., Kemker, R., Part, J.L., Kanan, C., Wermter, S.:
Continual lifelong learning with neural networks: A review.
Neural Netw. \textbf{113}, 54--71 (2019).
DOI: 10.1016/j.neunet.2019.01.012

\bibitem{paszke2019pytorch}
Paszke, A., et al.:
PyTorch: An imperative style, high-performance deep learning library.
In: Wallach, H. et al. (eds.) Advances in Neural Information Processing Systems 32, pp. 8024--8035 (2019)

\bibitem{shin2017continual}
Shin, H., Lee, J.K., Kim, J., Kim, J.:
Continual learning with deep generative replay.
In: Advances in Neural Information Processing Systems 30 (NeurIPS 2017), pp. 2990--2999 (2017).
DOI: 10.48550/arXiv.1705.08690

\bibitem{vitter1985random}
Vitter, J.S.:
Random sampling with a reservoir.
ACM Trans. Math. Softw. \textbf{11}(1), 37--57 (1985).
DOI: 10.1145/3147.3165

\bibitem{zenke2017synaptic}
Zenke, F., Poole, B., Ganguli, S.:
Continual learning through synaptic intelligence.
In: Precup, D., Teh, Y.W. (eds.) Proceedings of Machine Learning Research, vol. 70, pp. 3987--3995. PMLR (2017)

\end{thebibliography}
\end{document}